\documentclass[11pt]{article}

\usepackage[utf8]{inputenc} 
\usepackage[T1]{fontenc}
\usepackage{amsmath, amssymb, amsthm}
\usepackage{bm,bbm}
\usepackage{mathtools}
\usepackage{commath}
\usepackage{subfig}
\usepackage{graphicx}
\usepackage{indentfirst}
\usepackage{hyperref}
\usepackage{enumerate}
\usepackage{makecell}
\usepackage{fancyhdr}
\usepackage{setspace}
\usepackage[margin=2.5cm]{geometry}
\usepackage{tabularx}
\usepackage[title]{appendix}

\newtheorem{theorem}{Theorem}%  meant for continuous numbers
\newtheorem{lemma}[theorem]{Lemma}% 
\newtheorem{proposition}[theorem]{Proposition}% 

\newcommand{\R}{\mathbb{R}}
\newcommand{\Exp}{\mathbb{E}}
\newcommand{\Dcal}{\mathcal{D}}

\newcommand{\Fcal}{\mathcal{F}}

\newcommand{\Xcal}{\mathcal{X}}
\newcommand{\Ycal}{\mathcal{Y}}

\newcommand{\FR}{\mathrm{FR}}
\newcommand{\KL}{\mathrm{KL}}
\newcommand{\Hell}{\mathrm{H}}
\newcommand{\CE}{\mathrm{CE}}
\newcommand{\MAE}{\mathrm{MAE}}
\newcommand{\MSE}{\mathrm{MSE}}
\newcommand{\qCE}{{\text{$q$-}\mathrm{CE}}}

\newcommand{\diff}{\mathrm{d}}

\newcommand{\T}{\mathsf{T}}

\newcommand{\suchthat}{\mid}

\makeatletter
\let\@fnsymbol\@arabic
\makeatother
\newcommand*\samethanks[1][\value{footnote}]{\footnotemark[#1]}

\title{\textbf{The Fisher--Rao Loss for Learning under Label Noise}}
\author{Henrique~K.~Miyamoto\thanks{
	The authors are with the Institute of Mathematics, Statistics and Scientific Computing~(IMECC), University of Campinas~(Unicamp), Brazil. E-mail: \texttt{hmiyamoto@ime.unicamp.br}, \texttt{fabiom@ime.unicamp.br}, \texttt{sueli@unicamp.br}.
	}
	\and Fábio~C.~C.~Meneghetti\samethanks \and Sueli~I.~R.~Costa\samethanks}
\date{\vspace{-1.5em}}

\begin{document}

\maketitle

\begin{abstract}
	Choosing a suitable loss function is essential when learning by empirical risk minimisation.
	In many practical cases, the datasets used for training a classifier may contain incorrect labels, which prompts the interest for using loss functions that are inherently robust to label noise.
	In this paper, we study the Fisher--Rao loss function, which emerges from the Fisher--Rao distance in the statistical manifold of discrete distributions.
	We derive an upper bound for the performance degradation in the presence of label noise, and analyse the learning speed of this loss. Comparing with other commonly used losses, we argue that the Fisher--Rao loss provides a natural trade-off between robustness and training dynamics.
	Numerical experiments with synthetic and MNIST datasets illustrate this performance.
	
	\vspace{1em}
	
	\noindent \textbf{Keywords:} Classification, Fisher--Rao distance, information geometry, label noise, loss functions, neural networks
\end{abstract}

\section{Introduction} \label{sec:intro}

Supervised classification is an important problem in machine learning. Training a classifier (e.g., a deep neural network) can be done by empirical risk minimisation: a numerical optimisation algorithm is applied to find the model parameters that minimise the mean value of a loss function on the training dataset. Choosing a suitable loss function is essential, since different choices can affect the performance of the resulting classifier, as well as the training dynamics.

The output of a neural network trained for classification is often interpreted as giving a conditional probability distribution $p(y \rvert \pmb{x})$ of the class~$y$ given the input~$\pmb{x}$, which prompts the use of cross entropy as a loss function~\cite{bishop2006,goodfellow2016,calin2020}. Although originally used for regression problems, the mean squared error is also used as loss function, and several works have compared these two losses~\cite{kline2005,golik2013,janocha2016,demirkaya2020,hui2021}. Moreover, the design of new loss functions has been a topic of interest, and those are often tailored for specific problems or contexts, with many different inspirations, such as the correntropy similarity measure~\cite{singh2010}, the Wasserstein distance~\cite{frogner2015,hou2017}, and persistent homology~\cite{clough2020}.

A case of practical interest is when training datasets are corrupted with label noise, i.e., some of the class labels may be incorrect. This is a well-studied problem in machine learning: one of its sources is crowdsourcing labelling, and it can impact the performance of the generated model~\cite{frenay2014, sastry2017}. Many of the proposed solutions to mitigate this issue involve modifying the learning algorithms and have no theoretical guarantees of robustness. On the other hand, changing only the loss functions can be easily implemented and does not require further modifications to the structure of the classifiers. In addition, a sufficient condition for a loss function to be robust to label noise, i.e., for a classifier trained with that loss to achieve the same performance as if it was trained with clean data, was provided in~\cite{ghosh2015,ghosh2017}. These works also show that commonly used cross entropy and mean squared error losses do not satisfy this condition, while the mean absolute value loss does, but results in slow training. Other robust losses for label noise have been studied in~\cite{kumar2018,zhang2018}.

In this work, we study a loss function based on an information-geometric framework~\cite{amari1998,amari2000}. It is well known that the Fisher metric is the natural Riemannian metric for statistical manifolds, since it is essentially the only one that is invariant under sufficient statistics~\cite{ay2015}. In the case of discrete distributions, the geometry of the corresponding statistical manifold coincides with a spherical geometry; therefore, the Fisher--Rao distance between two points, as the geodesic distance induced by the Fisher metric, can be easily computed as the arc-length of an arc of great circle.
The proposed loss function for training classifiers is proportional to the squared Fisher--Rao distance in the manifold of discrete distributions, and has a twofold motivation: first, by the common interpretation that a classifier tries to minimise the error between predicted and ground truth class distributions, and, second, by the fact that the Fisher--Rao distance is a natural distance from the information-geometric viewpoint.

The Fisher--Rao distance has been recently used in unsupervised learning for shape clustering~\cite{gattone2017}, clustering financial returns~\cite{taylor2019} and image segmentation~\cite{pinele2020}. It has also found application as a regulariser term for adversarial learning~\cite{picot2021}, in detection of out-of-distribution samples~\cite{gomes2022}, and to unravel the geometry in the latent space of generative models~\cite{arvanitidis2022}. Here, we use the Fisher--Rao distance to construct a loss functions on its own, in a standard classification framework, and we compare it with other losses, particularly in the case of label noise. Specifically, we provide an upper-bound for its robustness against label noise, and analyse the learning speed using this loss. We argue that it provides a natural trade-off between these two aspects, and illustrate our theoretical analysis with numerical experiments.

Another well-known application of information geometry in deep learning consists in studying the geometry of neuromanifolds, i.e., the statistical manifolds formed by the joint distribution of the network inputs and outputs, parametrised by the network weights and biases~\cite{calin2020}. This gives rise to the celebrated natural gradient method~\cite{amari1998}, which modifies the gradient descent method by replacing the usual gradient with the gradient computed on this neuromanifold, using the Fisher information metric. We take a different approach in that we focus on the geometry of the outputs of the network, which corresponds to the manifold of discrete distributions, in order to also modify the learning algorithm, but by choosing a different loss function.

The remainder of this paper is organised as follows. Section~\ref{sec:inf-geom} reviews the geometry of the statistical manifold of discrete distributions. In Section~\ref{sec:learning}, the learning problem is described, and the proposed loss functions is analysed in terms of robustness to label noise and learning speed. Numerical results are presented and discussed in Section~\ref{sec:results}. Section~\ref{sec:conclusion} concludes the paper.

\section{Information Geometry of Discrete Distributions} \label{sec:inf-geom}

Let $(\Omega, \Fcal, \mu)$ be a $\sigma$-finite measure space and
\begin{equation*}
	M = \left\{ p_\theta = p(x;\theta) \suchthat \theta = (\theta_1, \dots, \theta_n) \in \Theta \subseteq \R^n \right\}
\end{equation*}
a parametric family of probability density functions $p_\theta \colon \Omega \to \R_+$ with respect to the measure~$\mu$. If $M$ is smoothly parametrised by $\theta \in \Theta \subset \R^n$ and satisfies certain regularity conditions~\cite{amari2000}, a Riemannian metric can be introduced on~$M$, called the \emph{Fisher metric}, given in matrix form by the \emph{Fisher information matrix} $G(\theta) = [g_{ij}(\theta)]_{i,j}$, with
\begin{equation}
	g_{ij}(\theta) = \Exp_{p_\theta} \!\! \left[ \frac{\partial \ell}{\partial \theta_i} \frac{\partial \ell}{\partial \theta_j} \right],
\end{equation}
for $1 \le i,j \le n$, where $\ell(\theta) \coloneqq \log p_\theta$ is the log-likelihood function and $\Exp_{p_\theta}[f] = \int_\Omega p_\theta(x) f(x)~\diff \mu(x)$ denotes the expectation of $f(x)$ with respect to $p_{\theta}$. Furthermore, considering the Kullback-Leibler divergence between two probability densities $p$ and $q$,
\begin{equation*}
	D_\KL (p\|q) = \int_{\Omega} p(x) \log \frac{p(x)}{q(x)}~\diff \mu(x),
\end{equation*}
the Fisher matrix can be also viewed as the Hessian of the Kullback-Leibler divergence, in the sense that
\begin{equation*}
	g_{ij}(\theta_0) = \left. \frac{\partial^2}{\partial \theta_i \partial \theta_j} D_\KL (p_{\theta_0} \| p_{\theta}) \right\vert_{\theta=\theta_0}.
\end{equation*}

With this structure, which is independent of the parametrisation $\theta \mapsto p_\theta$, $M$ is called a \emph{statistical manifold}. This is a natural choice for the geometry of probability distributions, as the Fisher metric is essentially the unique (up to a constant) Riemannian metric that is invariant under sufficient statistics on these manifolds~\cite{ay2015}.

Given a curve $\theta \colon \left[0,1\right] \to \Theta$ in the parameter space, the length of the corresponding curve $\gamma(t) = p_{\theta(t)}$ in the manifold $M$ is computed, in the Fisher geometry, as
\begin{equation*}
	l(\gamma) \coloneqq \int_0^1 \|\dot\gamma(t)\|_{G(\theta(t))}~\diff t
	= \int_0^1 \sqrt{ \dot\theta(t)^\T G (\theta(t)) \dot\theta(t)}~\diff t.
\end{equation*}
The \emph{Fisher--Rao distance} between two distributions $p_{\theta_1}$ and $p_{\theta_2}$ in $M$ is defined as the infimum of the length of all piecewise smooth paths linking them, i.e.,
\begin{equation*}
	d_\FR(p_{\theta_1}, p_{\theta_2}) \coloneqq \inf_{\gamma} \left\{ l(\gamma) \suchthat \gamma(0)=p_{\theta_1},\  \gamma(1)=p_{\theta_2} \right\}.
\end{equation*}
Moreover, when $M$ is a complete and connected manifold, this distance is realised by the length of a geodesic, called a \emph{minimising geodesic}. Unfortunately, closed-form expressions for the Fisher--Rao distance are only known in a few statistical manifolds, e.g., \cite{atkinson1981,calin2014,costa2015}.

In this paper, we are interested in the statistical manifold structure of discrete probability distributions. Given a sample space $\Omega = \left\{1, \dots, K \right\}$ and the set of functions $\delta^i \colon \Omega \to \{0,1\}$, defined by the Kronecker delta $\delta^i(j) \coloneqq \delta_{ij}$, we hereafter consider the statistical manifold of discrete distributions
\begin{equation*}
	M =
	\left\{ p = \textstyle\sum_{i=1}^K p_i \delta^i \suchthat p_i \in [0,1],\  \sum_{i=1}^K p_i =1 \right\},
\end{equation*}
which is in correspondence with the probability simplex
\begin{equation*}
	\Delta^{K-1} \coloneqq \left\{ \pmb{p}=(p_1, \dots, p_K) \suchthat p_i \in [0,1],\ \textstyle\sum_{i=1}^{K} p_i = 1 \right\}.
\end{equation*}
Furthermore, given the set of parameters
\begin{equation*}
	\Theta \coloneqq \left\{\theta = (p_1, \dots, p_{K-1}) \in \R^{K-1} \suchthat p_i \ge 0, \ \textstyle\sum_{i=1}^{K-1} p_i \le 1 \right\},
\end{equation*}
there is a standard parametrisation $\varphi\colon \Theta \to M$, of the form $\varphi(\theta) = \sum_{i=1}^K p_i \delta^i$, with $p_K \coloneqq 1- \sum_{i=1}^{K-1} p_i$. It is straightforward to show~\cite{calin2014} that, in this parametrisation, the Fisher information matrix has the form
\begin{equation*}
	g_{ij}(\theta) = \sum_{x \in \Omega} p(x) \frac{\partial \log p(x)}{\partial p_i} \frac{\partial \log p(x)}{\partial p_j} = \frac{\delta_{ij}}{p_i} + \frac{1}{p_K},
\end{equation*}
for $1 \le i,j \le K-1$.

In this case, there is a simpler way to obtain the geodesics than solving the geodesics differential equations. Considering the reparametrisation $z_i \coloneqq 2\sqrt{p_i}$, as in~\cite{kass1997,calin2014}, we have $\sum_{i=1}^{K} z_i^2 = 4$, so that the simplex $\Delta^{K-1~}$ is mapped to the part of the radius-two sphere with positive coordinates $S_{2,+}^{K-1} \coloneqq \left\{ \pmb{z} = (z_1,\dots,z_K) \in \R_+^{K} \suchthat \|\pmb{z}\|_2 = 2 \right\}$. The expression for the Fisher metric in this new parametrisation is
\begin{equation*}
	g_{ij} (\theta)
	= 4 \sum_{k=1}^K \pd{\sqrt{p_k}}{p_i} \pd{\sqrt{p_k}}{p_j}
	= \left\langle \frac{\partial \pmb z}{\partial p_i}, \frac{\partial \pmb z}{\partial p_j} \right\rangle.
\end{equation*}
Note that this is the formula of the standard spherical metric on $S_{2,+}^{K-1}$, i.e., the Euclidean metric in $\R^K$ restricted to $S_{2,+}^{K-1}$. As geodesics on the sphere are given by arcs of great circles, the length of a geodesic connecting $\pmb{z}_p$ and $\pmb{z}_q$ in $S_{2,+}^{K-1}$ (reparametrisation of the distributions $p = \textstyle\sum_{i=1}^K p_i \delta^i$ and $q = \textstyle\sum_{i=1}^K q_i \delta^i$, respectively) is double the angle $\alpha$ between these two vectors:
\begin{equation*}
	2\alpha
	= 2\arccos \left\langle \frac{\pmb z_p}{2}, \frac{\pmb z_q}{2} \right\rangle
	= 2\arccos \del[4]{ \sum_{i=1}^K \sqrt{p_i q_i} }.
\end{equation*}
Therefore, we obtain that the \emph{Fisher--Rao distance} between the discrete distributions $p = \textstyle\sum_{i=1}^K p_i \delta^i$ and $q = \textstyle\sum_{i=1}^K q_i \delta^i$ in $M$ is
\begin{equation} \label{eq:Fisher--Raor-dist}
	d_\FR (p,q) = 2\arccos \del[4]{ \sum_{i=1}^K \sqrt{p_i q_i} }.
\end{equation}

The reparametrisation to $S_{2}^{K-1}$ prompts an immediate approximation for the Fisher--Rao distance, given by the chordal distance, i.e., the Euclidean distance in $\R^K$, between $\pmb{z}_p$ and $\pmb{z}_q$:
\begin{equation*}
	\|\pmb z_p - \pmb z_q\|_2
	= 2\del[4]{ \sum_{i=1}^K ( \sqrt{p_i} - \sqrt{q_i} )^2 }^{1/2}.
\end{equation*}
This happens to be double the value of the the \emph{Hellinger distance}, a well-known distance between probability distributions~\cite[§~2.4]{tsybakov2009}:
\begin{equation} \label{eq:hellinger-dist}
	d_\Hell (p,q) = \del[4]{ \sum_{i=1}^{K} ( \sqrt{p_i} - \sqrt{q_i} )^2 }^{1/2},
\end{equation}
which is related to the Fisher--Rao distance by
\begin{equation} \label{eq:H-2sin-FR}
	d_\FR(p,q) = 4 \arcsin \left( \frac{d_\Hell(p,q)}{2} \right).
\end{equation}
Finally, note that, using the Taylor series expansion of the arc-sine function, we get the asymptotic approximation
\begin{equation} \label{eq:H-FR-approx}
	d_\FR(p,q) = 2 \, d_\Hell (p,q) + O\big( d_\Hell^3 (p,q)\big).
\end{equation}

\section{Learning for Classification} \label{sec:learning}

\subsection{Loss Functions}

We consider a supervised classification setting: $\pmb{x} \in \Xcal \subseteq \R^m$ denotes the feature vector, which belongs to exactly one class $y \in \Ycal \coloneqq \{ 1, \dots, K \}$, and we assume the data to follow a joint distribution $(\pmb{x}, y) \sim \Dcal$. Our goal is to learn a classifier (e.g., a neural network) $f\colon \Xcal \to \R^K$ that assigns to each input $\pmb{x}$ a vector of scores $\pmb{s} = (s_1, \dots, s_K) \coloneqq f(\pmb{x})$, which can be used to induce a hard decision $\hat{y} = \arg\max_{1 \le i \le K} s_i$. By applying the softmax function $\sigma \colon \R^K \to \Delta^{K-1}$ to the scores $\pmb{s}$, one obtains a conditional probability distribution $P(y \rvert \pmb{x})$ on~$\Ycal$, represented by the vector $\pmb{p} = (p_1, \dots, p_K) \coloneqq (\sigma \circ f)(\pmb{x}) = \sigma\left((s_1, \dots, s_K)\right)$, with $p_i = {e^{s_i}}/{\sum_{j=1}^{K} e^{s_j}}$, for $1 \le i \le K$.

Learning the model can be done by empirical risk minimisation, which is the population version of the problem $\min_{f} \mathbb{E}_{\Dcal}[L(y,f(\pmb{x}))]$ under a suitably chosen loss function $L \colon {\Ycal \times \R^K} \to \R_+$. In the following, we list some possible choices. Two popular criteria are mean squared error and the cross entropy loss functions. Let $\pmb{e}^{(y)} \in \R^K$ denote the one-hot vector that has value $1$ in the $y$-th coordinate, and value $0$ in the other coordinates.

The \emph{mean squared error}~(MSE) loss is defined as
\begin{equation} \label{eq:mse-loss}
	L_{\MSE}(y,f(\pmb{x}))
	\coloneqq \|\pmb{e}^{(y)} - (\sigma \circ f)(\pmb{x})\|_2^2
	= \|\pmb{p}\|_2^2 - 2p_y + 1,
\end{equation}
i.e., it is the squared Euclidean distance in the ambient space $\R^K$ restricted to $\Delta^{K-1}$.

The \emph{mean absolute error}~(MAE) loss\footnote{
	We adopt the coefficient $1/2$ to normalise it in the comparison with other loss functions. In this case, this loss coincides with the total variation distance~\cite[§~2.4]{tsybakov2009}.
} is defined as
\begin{equation} \label{eq:mae-loss}
	L_{\MAE}(y,f(\pmb{x}))
	\coloneqq \frac{1}{2}\|\pmb{e}^{(y)} - (\sigma \circ f)(\pmb{x})\|_1
	= 1-p_y,
\end{equation}
which is the taxicab distance in $\R^K$ restricted to $\Delta^{K-1}$.

The \emph{cross entropy}~(CE) loss is given by
\begin{equation} \label{eq:ce-loss}
	L_{\CE}(y,f(\pmb{x})) \coloneqq - \log \left( \pi_y \left( (\sigma \circ f)(\pmb{x}) \right) \right) = - \log  p_y,
\end{equation}
where $\pi_y(\pmb{v})$ denotes the projection of vector $\pmb{v}$ to its $y$-th component. Note that ${L}_{\mathrm{CE}}(y,f(\pmb{x})) = D_\KL ( \pmb{e}^{(y)} \| (\sigma \circ f)(\pmb{x}) )$, i.e., in this case, the cross entropy is equal to the Kullback-Leibler divergence between the true conditional distribution $\pmb{e}^{(y)}$, which is in fact deterministic, and the predicted one $\pmb{p} = (\sigma \circ f)(\pmb{x})$.

An interesting generalisation of the cross entropy loss can be obtained by replacing the logarithm in \eqref{eq:ce-loss} with the Tsallis $q$-logarithm\footnote{
	In this paper, $\log_q$ denotes the Tsallis $q$-logarithm, not to be confused with the logarithm to the base $q$.
}~\cite{tsallis1994}:
\begin{equation} \label{eq:q-logarithm}
	\log_q (x) \coloneqq
	\begin{cases}
		\frac{x^{1-q} - 1}{1-q}, &q\neq 1\\
		\log x, &q=1
	\end{cases},
\end{equation}
% which satisfies $\lim_{q\to 1} \log_q(x) = \log(x)$
defined for $x>0$ and $q \in \R$. In particular, for $q \in \left[0, 1\right)$, this is the loss function proposed in~\cite{zhang2018}, where it was defined in terms of the negative Box-Cox transformation\footnote{
	It should be noted that our parameter~$q$ used in \eqref{eq:q-loss} corresponds to the parameter $1-q$ used in~\cite{zhang2018}.
}, and given by
\begin{equation} \label{eq:q-loss}
	L_{\qCE}(y,f(\pmb{x})) \coloneqq - \log_{q} \left( \pi_y \left((\sigma \circ f)(\pmb{x}) \right) \right) = - \log_{q} p_y ,
\end{equation}
to which we refer as \emph{$q$-cross entropy} loss ($q$-CE). Interestingly, as remarked in~\cite{zhang2018}, for $q=0$, it coincides with the MAE loss, and, for $q=1$, with the CE loss.

The main contribution of this work is the study of the Fisher--Rao loss function, inspired by an information-geometric framework, cf. Section~\ref{sec:inf-geom}. This will be defined, up to multiplicative constant, as the square of the Fisher--Rao distance in the manifold of discrete distributions. Taking the square is reasonable, as other loss functions, such as the MSE and $q$-CE losses, behave, at least locally, as squared distances. Using~\eqref{eq:Fisher--Raor-dist}, we can compute
\begin{equation*}
	d^2_{\FR} \big( \pmb{e}^{(y)}, \pmb{p} \big) 
	= 4 \left(\arccos \sqrt{ \pi_y (\pmb{p})} \right)^2 \nonumber
	= 4 \left( \arccos \sqrt{p_y} \right)^2.
\end{equation*}
Therefore, disregarding the multiplicative constant---which is fine, since multiplying the loss is equivalent to multiplying the learning rate---, we define the \emph{Fisher--Rao loss} as
\begin{equation} \label{eq:fr-loss}
	L_{\FR}(y,f(\pmb{x}))
	\coloneqq \frac{1}{4} d^2_{\FR} \big(\pmb{e}^{(y)}, (\sigma \circ f)(\pmb{x}) \big) 
	= \left( \arccos \sqrt{p_y}\right)^2.
\end{equation}

We advocate that, conceptually, the Fisher--Rao loss is a natural choice of loss function, since it emerges from the Fisher metric in the statistical manifold of discrete distributions, cf. Section~\ref{sec:inf-geom}. Furthermore, using the Taylor series expansion, it can be verified that the $q$-CE loss (which includes the CE and MAE losses) has an asymptotic approximation in terms of the Fisher--Rao distance, for small values of the distance to $\pmb{e}^{(y)}$:
\begin{equation*}
	L_\qCE \left( y,f(\pmb{x}) \right)
	= \frac{1}{4}d_\FR^2 \big( \pmb{e}^{(y)}, \pmb{p} \big)
	+ O \left( d_\FR^4(\pmb{e}^{(y)}, \pmb{p}) \right).
\end{equation*}

Finally, in view of the arc-chord approximation~\eqref{eq:H-FR-approx}, we can also consider a loss function given by the square of the Hellinger distance, as mentioned in \cite[§~3.10]{calin2020}. Using~\eqref{eq:hellinger-dist}, we thus define the \emph{Hellinger loss} as
\begin{equation} \label{eq:h-loss}
	L_{\Hell} (y,f(\pmb{x}))
	\coloneqq d^2_{\Hell} \big( \pmb{e}^{(y)}, (\sigma \circ f)(\pmb{x}) \big)
	= \sum_{i=1}^{K} {\left( \sqrt{\pi_i (\pmb{e}^{(y)})} - \sqrt{p_i} \right)}^2
	= 2 \, (1- \sqrt{p_y}).
\end{equation}
Note that this is also a particular case of the $q$-cross entropy loss, for $q=1/2$, and that it also provides an approximation for the Fisher--Rao loss.

The following proposition establishes asymptotic relations between the Fisher--Rao and other loss functions.

\begin{proposition} \label{prop:losses}
	Let $L_\MAE$, $L_\CE$, $L_\qCE$, $L_\FR$ and $L_\Hell$ denote the loss functions defined, respectively, in \eqref{eq:mae-loss}, \eqref{eq:ce-loss}, \eqref{eq:q-loss}, \eqref{eq:fr-loss} and \eqref{eq:h-loss}.
	\begin{enumerate}
		\item \label{item:i} We have the asymptotic approximation:
		\begin{equation*}
			L_{\FR} (y,f(\pmb{x})) = L_{\qCE}(y,f(\pmb{x})) + O \big( L_{\qCE}^{2}(y,f(\pmb{x})) \big).
		\end{equation*}
		In particular, respectively for $q=0$, $q=1$ and $q=1/2$, we have:
		\begin{align*}
			L_{\FR}(y,f(\pmb{x})) &= L_{\MAE}(y,f(\pmb{x})) + O \big( L_{\MAE}^{2}(y,f(\pmb{x})) \big),\\
			L_{\FR}(y,f(\pmb{x})) &= L_{\CE}(y,f(\pmb{x})) + O \big( L_{\CE}^{2}(y,f(\pmb{x})) \big),\\
			L_{\FR}(y,f(\pmb{x})) &= L_{\Hell}(y,f(\pmb{x})) + O \big( L_{\Hell}^{2}(y,f(\pmb{x})) \big).
		\end{align*}
		
		\item \label{item:ii}	Moreover, the following inequality chain holds:
		\begin{equation*}
			L_{\MAE}(y,f(\pmb{x})) \le L_{\Hell}(y,f(\pmb{x})) \le L_{\FR}(y,f(\pmb{x})) \le L_{\CE}(y,f(\pmb{x})).
		\end{equation*}
	\end{enumerate}
\end{proposition}

\begin{proof}
	In the following, we omit the argument of the losses, for notational simplicity.
	We shall first prove item~\ref{item:i} for $q \in \left[0,1\right)$. Isolating $p_y$ in \eqref{eq:q-loss} and substituting it in~\eqref{eq:fr-loss}, we obtain
	\begin{equation*}
		L_\FR = \left[ \arccos \left( \left(1-(1-q)L_{\qCE} \right)^{\frac{1}{2(1-q)}} \right) \right]^2.
	\end{equation*}
	The first-order Taylor series expansion of $g_1(x) \coloneqq \big[ \arccos \big( \left(1-(1-q) x \right)^{\frac{1}{2(1-q)}} \big) \big]^2$ around $x=0$ is $g_1(x) = x + O(x^2)$, yielding the desired result.
	Now, we prove it for $q=1$. Again, isolating $p_y$ in \eqref{eq:q-loss} and substituting it in~\eqref{eq:fr-loss}, we obtain
	\begin{equation*}
		L_\FR = \left( \arccos \sqrt{e^{-L_\CE}} \right)^2.
	\end{equation*}
	Analogously, the first-order Taylor series of $g_2(x) \coloneqq ( \arccos \sqrt{e^{-x}} )^2$ around $x=0$ is $g_2(x) = x + O(x^2)$, completing the proof of the first item.
	
	For the first inequality of item~\ref{item:ii}, note that $L_{\MAE} = (1+\sqrt{p_y})(1-\sqrt{p_y}) \le 2(1-\sqrt{p_y}) = L_{\Hell}$, since $p_y \le 1$. The second inequality is immediate from~\eqref{eq:H-2sin-FR}, i.e., the fact that the chordal distance is less than the geodesic distance on the sphere. For the third inequality, note that the first-order Taylor expansion of $g_2(x)$ around zero can be written as $g_2(x)=x+R_1$, with
	\begin{equation*}
		R_1 = \frac{x^2}{2}g''(x^{*}) = \frac{x^2}{4} \left( \frac{\sec^2 \sqrt{x^*}}{x^*} - \frac{\tan \sqrt{x^*}}{(x^*)^{3/2}} \right) \ge 0,
	\end{equation*}
	for $x^{*} \in (0,x) \subseteq (0,{\pi^2}/4)$. Therefore $g_2(x) \ge x$, and hence $L_{\CE} \ge L_{\FR}$.\qedhere
\end{proof}

These results state that the $q$-CE loss---and, in particular, MAE, CE and Hellinger loss functions---can be seen as first-order asymptotic approximations for the Fisher--Rao loss. However, while the approximations are valid for small values of the losses, the functions may behave largely differently for higher values.

\subsection{Robustness to Label Noise}\label{sec:robustness}

Let us now consider the problem of learning in the presence of label noise, cf.~\cite{ghosh2015,ghosh2017}. In this case, the classifier does not have access to a set of \emph{clean} samples $\left\{ (\pmb{x}_i, y_i) \right\}_{i=1}^{N}$, but instead to a set of \emph{noisy} data $\left\{ (\pmb{x}_i, \tilde{y}_i) \right\}_{i=1}^{N}$, where $\tilde{y}_i$ denotes the noisy labels. We suppose the noisy data follow a joint distribution $(\pmb{x}, \tilde{y}) \sim \Dcal_\eta$.

Here, we focus on the case of \emph{uniform} or \emph{symmetric} label noise, i.e., when the noisy label $\tilde{y}_i$ does not depend on the feature vector $\pmb{x}_i$, nor on the true label $y_i$. In this case, the label noise is modelled by
\begin{equation}
	\Pr(\tilde{y}_i = j \vert y_i = k) =
	\begin{cases}
		1 - \eta, & j=k\\
		\frac{\eta}{K-1}, &j \neq k
	\end{cases},
\end{equation}
for some constant $\eta \in \left[ 0,1 \right]$, called \emph{noise rate}.

We denote $R_L(f) \coloneqq \Exp_{\Dcal}\left[ L({y}, f(\pmb{x})) \right]$ and $R^{\eta}_L(f) \coloneqq \Exp_{\Dcal_\eta}\left[ L(\tilde{y}, f(\pmb{x})) \right]$ the risks with respect to clean and noisy data, respectively. Moreover, let $f^*$ and $\hat{f}$ be global minimisers of $R_L(f)$ and $R^\eta_L(f)$, respectively. The risk minimisation under loss function $L$ is said to be \emph{noise tolerant}~\cite{manwani2013,ghosh2015} if $\hat{f}$ has the same probability of misclassification as that of $f^*$. In other words, the performance of the classifier trained with the noisy dataset is as good as it would be had it been trained with a clean dataset.

An important contribution from~\cite[Theorem~1]{ghosh2017} was to provide a sufficient condition for a loss to be robust under label noise: a loss function $L$ is tolerant under uniform label noise with $\eta < \frac{K-1}{K}$, if it satisfies $\sum_{i=1}^K L(i,f(\pmb{x})) = C$, for all $\pmb{x} \in \Xcal$ and all classifiers $f$. It is shown that the MAE loss satisfies this condition, while the MSE and CE losses do not.
However, even for loss functions not satisfying it, if the quantity $\sum_{i=1}^K L(i,f(\pmb{x}))$ is bounded, it is still possible to obtain performance guarantees for robustness to label noise. This has been done for the $q$-cross entropy loss in~\cite[Theorem~1]{zhang2018}. Inspired by that, we derive a similar result for the Fisher--Rao loss function.

\begin{lemma}\label{lem:fr-loss-sum-bound}
	The Fisher--Rao loss $L_{\FR}$ satisfies
	\begin{equation}\label{eq:fr-loss-sum-bound}
		K \left( \arccos \frac{1}{\sqrt{K}} \right)^2
		\leq
		\sum_{i=1}^K L_{\FR}(i,f(\pmb{x}))
		\leq
		\frac{\pi^2}{4}(K-1).
	\end{equation}
\end{lemma}

\begin{proof}
	We shall use Lagrange multipliers to find the critical points of $F(\pmb{p}) \coloneqq \sum_{j=1}^K (\arccos \sqrt{p_j})^2$ subject to $G(\pmb{p}) \coloneqq \sum_{j=1}^K p_j - 1 = 0$. We have $\frac{\partial F}{\partial p_i} (\pmb{p}) = - \frac{\arccos{\sqrt{p_i}}}{ \sqrt{ p_i (1-p_i) } }$ and $\frac{\partial G}{\partial p_i} (\pmb{p}) = 1$, therefore critical points in the interior of the simplex $\Delta^{K-1}$ are of the form $\frac{\partial F}{\partial p_i} (\pmb{p}) = \lambda$, for $1 \le i \le K$. Since $x \mapsto -\frac{\arccos\sqrt{x}}{\sqrt{x(1-x)}}$ is injective, this happens if, and only if, $p_1 = p_2 = \cdots = p_K = \frac{1}{K}$, which leads to $F(\pmb{p}) = K \left( \arccos \frac{1}{\sqrt{K}} \right)^2$.
	On the other hand, if there are critical points on the boundary of $\Delta^{K-1}$, they will have a zero in some position, e.g., $\pmb{p}=(0,p_2,\dots,p_K)$. Repeating the same argument, the critical point will have the form $p_1 = 0$, $p_2 = \cdots = p_K= \frac{1}{K-1}$, so that $F(\pmb{p}) = \frac{\pi^2}{4} + (K-1) \left( \arccos \frac{1}{\sqrt{K-1}} \right)^2$. By induction, any critical value of $F$ has the form $F(\pmb{p}) = (K-j) \frac{\pi^2}{4} + j \left(\arccos\frac{1}{\sqrt{j}}\right)^2$, for $1 \le j \le K$. By comparing the values of $F$ in all these points, we find maximum and minimum, respectively, for $j=1$ and $j=K$, which yields the desired inequalities. \qedhere
\end{proof}

\begin{proposition} \label{prop:bounds-fr}
	For the Fisher--Rao loss function~$L_\FR$, under uniform label noise with $\eta < \frac{K-1}{K}$, we have
	\begin{equation}\label{eq:A-ineq}
		0 \le R^\eta_{L_\FR}(f^*) - R^\eta_{L_\FR}(\hat f) \le A_\FR
	\end{equation}
	and
	\begin{equation}\label{eq:A'-ineq}
		B_\FR \le R_{L_\FR}(f^*) - R_{L_\FR}(\hat f) \le 0,
	\end{equation}
	where
	\begin{equation}
		A_\FR \coloneqq A_\FR(K,\eta) \coloneqq \eta \left( \frac{\pi^2}{4} - \frac{K}{K-1} \left(\arccos\frac{1}{\sqrt{K}}\right)^2 \right)
	\end{equation}
	and
	\begin{eqnarray}
		B_\FR \coloneqq B_\FR(K,\eta) \coloneqq \eta\frac{ K (\arccos \frac{1}{\sqrt{K}} )^2 - \frac{\pi^2}{4} (K-1) }{K -1 -\eta K}.
	\end{eqnarray}
\end{proposition}
\begin{proof}
	The inequalities involving $0$ in \eqref{eq:A-ineq} and \eqref{eq:A'-ineq} are straightforward consequences of the definitions of $f^*$ and $\hat f$. We shall prove the inequalities involving $A_\FR$ and $B_\FR$. For \eqref{eq:A-ineq}, as in the proof of \cite[Theorem~1]{zhang2018}, we have that
	\begin{align*}
		R^\eta_{L_\FR}(f)
		&= \Exp_{(\pmb{x}, \tilde y)} \left[ L_\FR(\tilde y, f(\pmb{x})) \right]
		= \Exp_{\pmb{x}} \Exp_{y\rvert \pmb{x}} \Exp_{\tilde y \rvert \pmb{x},y} \left[ L_\FR(\tilde y, f(\pmb{x})) \right]\\
		&= \Exp_{\pmb{x}} \Exp_{y\rvert \pmb{x}} \left[ (1-\eta) L_\FR (y,f(\pmb{x})) + \frac{\eta}{K-1} \sum_{i \neq y} L_\FR(i, f(\pmb{x})) \right]\\
		&=  \left( 1 - \frac{\eta K}{K-1} \right) R_{L_\FR} (f)
		+ \frac{\eta}{K-1} \Exp_{\pmb{x}} \Exp_{y \rvert \pmb{x}} \left[ \sum_{i=1}^K L_{\FR} (i, f(\pmb{x})) \right].
	\end{align*}
	
	Using Lemma \ref{lem:fr-loss-sum-bound}, we obtain
	\begin{equation}\label{eq:RLFR1}
		\del{ 1 - \frac{\eta K}{K-1} } R_{L_\FR} (f^*)
		+ \frac{\eta K}{K-1} \left( \arccos\frac{1}{\sqrt{K}} \right)^2
		\le R_{L_\FR}^\eta (f^*),
	\end{equation}
	and
	\begin{equation}\label{eq:RLFR2}
		R_{L_\FR}^\eta (\hat f) \le
		\del{ 1 - \frac{\eta K}{K-1} } R_{L_\FR} (\hat f)
		+ \frac{\eta \pi^2}{4}.
	\end{equation}
	Subtracting these two inequalities, we get
	\begin{equation*}
		R^\eta_{L_\FR} (f^*) - R^\eta_{L_\FR}(\hat f) \le
		\left( 1-\frac{\eta K}{K-1} \right)
		\left( R_{L_\FR} (f^*) - R_{L_\FR}(\hat f) \right)
		+
		A_\FR.
	\end{equation*}
	Using that $R_{L_\FR}(f^*) \le R_{L_\FR}(\hat f)$ and $\eta < \frac{K-1}{K}$ yields $R^\eta_{L_\FR} (f^*) - R^\eta_{L_\FR}(\hat f) \le A_\FR$.
	
	For \eqref{eq:A'-ineq}, we use the same technique, but isolating $R_{L_\FR}(f)$ in \eqref{eq:RLFR1} and \eqref{eq:RLFR2}, with $f^*$ and with $\hat f$ swapping places, thus obtaining
	\begin{equation*}
		R_{L_\FR} (\hat f) \le
		\frac{
			R^\eta_{L_\FR}(\hat f) - \frac{\eta K}{K-1} \left( \arccos\frac{1}{\sqrt{K}} \right)^2
		}{
			1 - \frac{\eta K}{K-1},
		}
	\end{equation*}
	and
	\begin{equation*}
		\frac{
			R^\eta_{L_\FR}(f^*) - \frac{\eta \pi^2}{4}
		}{
			1 - \frac{\eta K}{K-1}
		}
		\le R_{L_\FR} (f^*).
	\end{equation*}
	When the two inequalities are subtracted, we get
	\begin{equation*}
		R_{L_\FR}(f^*) - R_{L_\FR}(\hat{f}) \ge \frac{(K-1) \left( R_{L_\FR}^\eta(f^*) - R_{L_\FR}^\eta(\hat{f}) \right)}{K-1-\eta K} + B_\FR.
	\end{equation*}
	Using that $R^\eta_{L_\FR}(\hat f) \le R^\eta_{L_\FR}(f^*)$ and $\eta < \frac{K-1}{K}$ yields $B_\FR \le R_{L_\FR}(f^*) - R_{L_\FR}(\hat f)$.
\end{proof}

This result provides an upper bound for the performance degradation of a classifier trained with the Fisher--Rao loss under uniform label noise, in terms of the noise rate~$\eta$ and the number of classes~$K$. This shows that this type of noise can only have a limited impact in the performance of the trained classifier.

Table~\ref{tab:bounds-noise} compares the analogous bounds for other loss functions. The results for the MSE loss are derived in Appendix~\ref{ap:MSE-loss}. For the MAE loss, the bounds are zero, since it is proved to be robust to uniform label noise~\cite{ghosh2017}. For the CE loss, no such bounds can be obtained, as this loss in unbounded, so its entries are marked as $+\infty$ and $-\infty$ on the table. The table also contains the result for the $q$-CE losses, which was taken from \cite{zhang2018}, and in particular for the Hellinger loss ($q=1/2$).

\begin{table}
	%	\vspace{1.5em}
	\centering
	\caption{Bounds $A(K,\eta)$ and $B(K,\eta)$ for different loss functions.}
	\label{tab:bounds-noise}
	\renewcommand{\arraystretch}{2}
	\begin{tabular}{cccc}
		\hline
		Loss function & $A(K,\eta)$ & $B(K,\eta)$ \\ \hline \hline
		Mean squared error (MSE)  & $\eta$ & $-\eta\frac{K-1}{K-1-\eta K}$ \\ 
		Mean absolute error (MAE) & $0$ & $0$ \\ 
		Cross entropy (CE)        & $+\infty$ & $-\infty$ \\ 
		$q$-cross entropy ($q$-CE)        & $\eta \frac{K^q - 1}{(1-q)(K-1)}$ & $\eta \frac{1-K^q}{(1-q)(K-1-\eta K)}$ \\ 
		Fisher--Rao                & $\eta \left( \frac{\pi^2}{4} - \frac{K}{K-1} \left(\arccos\frac{1}{\sqrt{K}}\right)^2 \right)$ & $\eta \frac{K \left( \arccos\frac{1}{\sqrt{K}} \right)^2 - \frac{\pi^2}{4}\left(K-1\right)}{K-1-\eta K}$ \\ 
		Hellinger ($q=1/2$)       & $\eta \frac{2(\sqrt{K} - 1)}{K-1}$ & $\eta \frac{2(1-\sqrt{K})}{(K-1-\eta K)}$ \\ \hline
	\end{tabular}
\end{table}

To further investigate the behaviour of the bounds that depend on the noise rate~$\eta$ and the number of classes~$K$, we write the noise rate as $\eta = \alpha \left( 1 - \frac{1}{K} \right)$, parametrised by $\alpha \in \left[0,1\right)$. Varying this value, the noise rate goes from no label noise to the maximum allowed for a given number of classes~$K$.
Figure~\ref{fig:bounds-alpha} illustrates how the bounds $A(K,\eta)$ and $B(K,\eta)$ vary with the value $\alpha$ for different loss functions and $K=10$. For the $q$-CE loss, we choose the value $q=0.7$, which has been used in the experiments in~\cite{zhang2018}. We see that, as $\alpha \to 1$, the values of $A(K,\eta)$ tend each to a constant, while those of $B(K,\eta)$ go to~$-\infty$. 

\begin{figure}
	\centering
	\includegraphics[width=\linewidth]{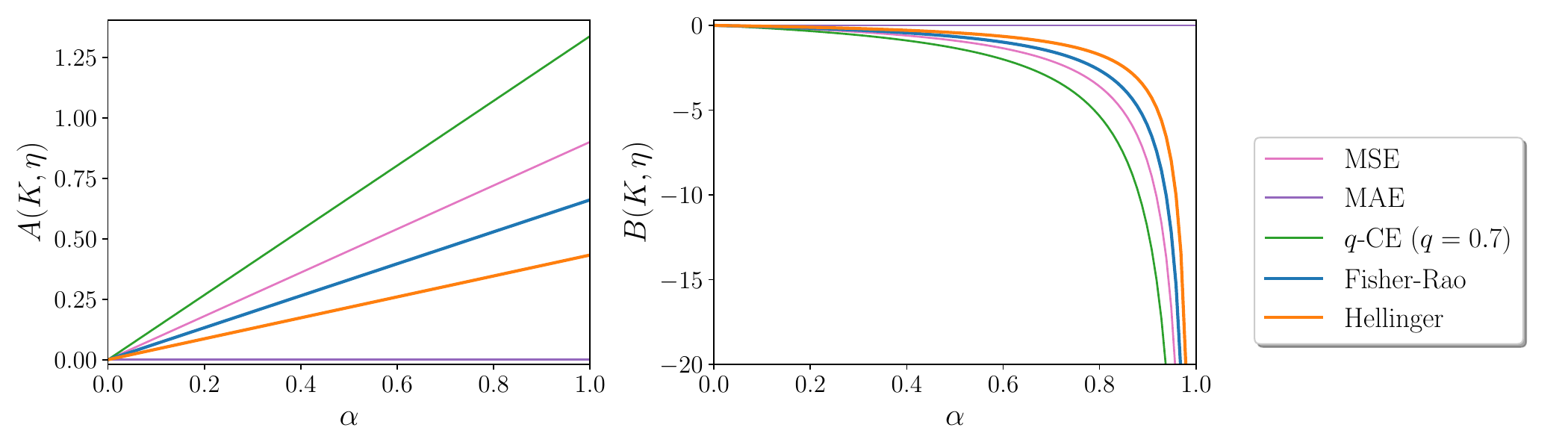}
	\caption{Bounds $A(K,\eta)$ and $B(K,\eta)$, for $K=10$ and $\eta = \alpha \left(1-\frac{1}{K}\right)$, as function of $\alpha \in \left[0,1\right)$.}
	\label{fig:bounds-alpha}
\end{figure}

To study the dependence on the number of classes~$K$, we plot the bounds versus~$K$, for $\eta = 0.8 \left(1-\frac{1}{K}\right)$, in Figure~\ref{fig:bounds-k}. We see that, for Fisher--Rao, Hellinger and $q$-cross entropy losses, after a peak, the bounds get tighter as the number of classes increases. Indeed, for the Fisher--Rao bounds, we have
\begin{equation*}
	\lim_{K\to\infty} A_\FR\left(K, \alpha\left( 1-\frac{1}{K} \right)\right)=0
	\quad\text{and}\quad
	\lim_{K\to\infty} B_\FR\left(K, \alpha\left( 1-\frac{1}{K} \right)\right)=0.		
\end{equation*}
This means that the effect of the label noise gets `diluted' as the number of classes grows, for any $\alpha \in \left[0,1\right)$. Even though this effect may sound intuitive, this is not what happens with the bounds for the MSE loss, which tend to a constant value as $K$ increases.

\begin{figure}
	\centering
	\includegraphics[width=\linewidth]{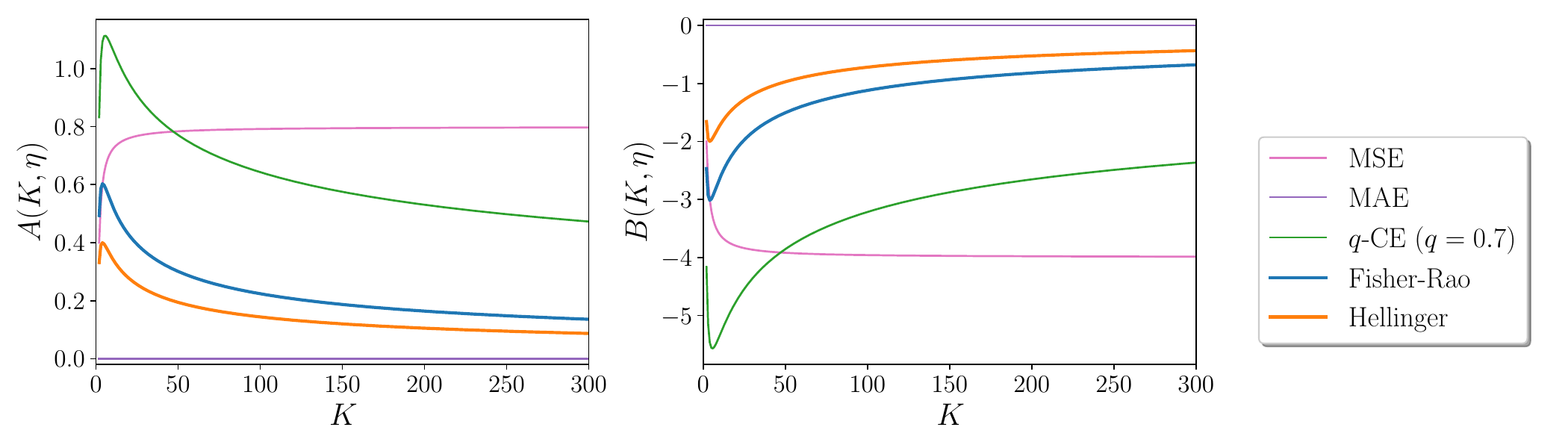}
	\caption{Bounds $A(K,\eta)$ and $B(K,\eta)$ for $\eta = 0.8 \left(1-\frac{1}{K}\right)$, as function of ${2 \le K \le 300}$.}
	\label{fig:bounds-k}
\end{figure}

\subsection{Learning Speed}

Learning in neural networks involves minimising the empirical risk~$\bar{R}_L = \frac{1}{N}\sum_{i=1}^{N} L(y_i,f(\pmb{x}_i))$ by means of a numerical algorithm, such as gradient descent. The gradient vector $\nabla_\omega \bar{R}_L = \frac{1}{N} \sum_{i=1}^{N} \nabla_{\omega} L(y_i,f(\pmb{x}_i)) $ with respect to the network parameters $\omega$ gives the direction and magnitude of each updating step towards the local minimum of $\bar{R}_L$. To investigate the learning speed obtained with each loss, we shall study the magnitude of this gradient vector.

An interesting feature of MAE, CE, $q$-CE, Fisher--Rao and Hellinger losses is that they can all be written as to only depend on the probability $p_y = \sigma(s_y)$ that the model assigns to the correct class. The MSE loss, on the other hand, depends on all coordinates of the vector $\pmb{p} = (\sigma \circ f)(\pmb{x})$, meaning that it tries to simultaneously make the predicted probability $p_y$ of the correct class close to~$1$ and to minimise the probability $p_i$ of all other classes $i \neq y$.

Specifically, the MAE, CE, $q$-CE Fisher--Rao and Hellinger losses can be written as functions of the form
\begin{equation}\label{eq:Loss-function-form}
	L(y,f(\pmb{x}))
	= h \left( \pi_y \left( (\sigma\circ f)(\pmb{x}) \right) \right)
	= h \left( p_y \right),
\end{equation}
for some monotonically decreasing differentiable function $h \colon [0,1] \to \R$ with $h(1)=0$. In this case, we have
\begin{equation*}
	\nabla_\omega L(y,f(\pmb{x})) = h'(p_y) \nabla_\omega \pi_y \left( (\sigma \circ f)(\pmb{x}) \right).
\end{equation*}
Since the term $\nabla_\omega \pi_y \left( (\sigma \circ f)(\pmb{x}) \right)$ does not depend on the loss function, it is enough to study the absolute value of the multiplicative factor $\vert h'(p_y) \vert$ to compare the magnitude of the gradient for different losses. The particular functions $h$ and their derivatives are given in Table~\ref{tab:functions-h} and plotted in Figure~\ref{fig:losses}. Note that this plot recovers the inequality chain obtained in Proposition~\ref{prop:losses}, item~\ref{item:ii}.

As it has been previously noted~\cite{kumar2018,zhang2018}, the fact that the gradient of the MAE loss is constant is the main reason why this loss results in slow training, in spite of its robustness to label noise. On the other extreme, with the CE loss, the magnitude of $\vert h'(p_y) \vert$ increases as $p_y$ decreases, meaning that the farther away from the local minimum (i.e., prediction is more wrong), the larger the steps in the gradient descent method will be. Also, the form of its derivative implies that more emphasis is put on samples whose predicted class is different from the assigned label (noisy or clean), which explains both the fast convergence speed and the weak tolerance to label noise.

The magnitude of $\vert h'(p_y) \vert$ for the $q$-CE, Fisher--Rao and Hellinger loss functions similarly depend on the value of $p_y$, indicating they have better training dynamics than the MAE loss. Interestingly, the order of losses with higher magnitude of $\vert h'(p_y) \vert$, suggesting faster convergence, is the opposite of that observed for robustness (CE, Fisher--Rao, Hellinger, MAE) in Figure~\ref{fig:bounds-k}. Thus, the Fisher--Rao and the Hellinger loss can be seen as providing a trade-off between learning dynamics and robustness against label noise. Particularly, these results indicate that the Fisher--Rao loss can provide a modest improvement in learning speed, at the cost of a modest reduction of robustness, when compared to the Hellinger loss.

\begin{figure}
	\centering
	\includegraphics[width=\linewidth]{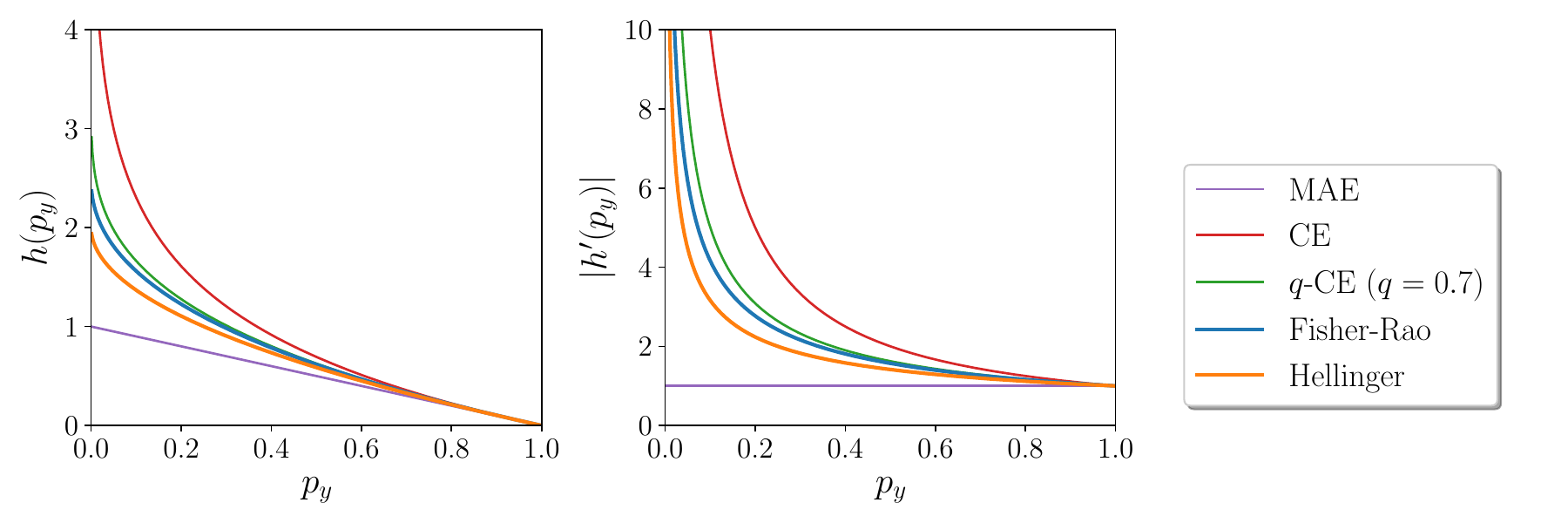}
	\caption{Functions $h(p_y)$ and their derivatives $\vert h'(p_y) \vert$ for different loss functions, cf.~Table~\ref{tab:functions-h}.}
	\label{fig:losses}
\end{figure}

\begin{table}
	\centering
	\caption{Functions $h(p_y)$ and their derivatives $\vert h'(p_y) \vert$.}
	\label{tab:functions-h}
	\renewcommand{\arraystretch}{1.4}
	\setlength{\tabcolsep}{10pt}
	\begin{tabular}{cccc}
		\hline
		Loss function              &  $h(p_y)$ & $\vert h'(p_y) \vert$	\\ \hline \hline
		Mean absolute error (MAE)  & $1-p_y$ & $1$ \\ 
		Cross entropy (CE)         & $- \log p_y$ & $\frac{1}{p_y}$	\\
		$q$-cross entropy ($q$-CE) & $- \log_q p_y$ & $\frac{1}{{(p_y)}^q}$ \\
		Fisher--Rao                 & $\left( \arccos\sqrt{p_y} \right)^2$ & $\frac{\arccos\sqrt{p_y}}{\sqrt{p_y(1-p_y)}}$	\\
		Hellinger ($q=1/2$)        & $2 \left( 1-\sqrt{p_y} \right)$ & $\frac{1}{\sqrt{p_y}}$ \\ \hline
	\end{tabular}
\end{table}

\section{Experimental Results} \label{sec:results}

We conduct numerical experiments to study the performance of different loss functions, on a synthetic and a real dataset. The experiments use simple neural networks and are intended to illustrate the theoretical results derived in the previous sections and showcase the trade-off provided by the Fisher--Rao loss, rather than to generate state-of-the-art results. We compare the Fisher--Rao loss function with CE and Hellinger loss, which approach the former, cf.~Prosposition~\ref{prop:losses}, and also with the MSE loss, which is commonly used.

We assess the performances both on the original (noiseless) datasets and in the presence of uniform label noise with noise rate~$\eta$. Only the training datasets are corrutpted with noise, while the testing datasets are clean. In all experiments the activation function is ReLU, and stochastic gradient descent is used for training the model. The learning rate for each loss was hand-tuned by grid search to provide the best performance.

\subsection{Synthetic Dataset}

First we consider a synthetic dataset formed by 8,000 training examples and 2,000 testing examples of $\text{100}$-dimensional vectors divided into 10~classes. The data are generated by Gaussian distributions centred on the vertices of hyper-cube, using the method \texttt{make\_classification} from scikit-learn~\cite{scikit-learn}. We set up a multilayer perceptron~(MLP) network with three hidden layers with 80, 40 and 20 neurons, respectively. Batch size is~20 and the model is trained for 20~epochs. The experiments were repeated five times, and we report the averages and standard deviations.

We run tests for noiseless data and with uniform label noise with $\eta=0.3$ and $\eta=0.5$. The evolution of the accuracy on the training (dashed lines) and testing (continuous lines) datasets are shown in Figure~\ref{fig:synthetic-results}, where the mean results are represented by opaque lines and the shaded regions indicate the standard deviation. The final test accuracies are summarised in Table~\ref{tab:synthetic-results}, where the two best accuracies are boldfaced.

\begin{table}[h]
	\begin{center}
		\begin{minipage}{\textwidth}
			\caption{Test accuracy (\%) for synthetic dataset.}
			\label{tab:synthetic-results}
			\centering
			\setlength{\tabcolsep}{4pt}
			\renewcommand{\arraystretch}{1.15}
			\begin{tabular}{@{}ccccc@{}}
				\hline
				Loss  & Noiseless & $\eta = 0.3$ & $\eta = 0.5$ \\
				\hline
				Mean square error    & 88.39 ($\pm$0.70) & 74.43 ($\pm$0.41) & 64.08 ($\pm$0.70) \\
				Cross entropy    & \textbf{90.21 ($\pm$1.27)} & 73.68 ($\pm$0.99) & 60.78 ($\pm$1.15) \\
				Fisher--Rao   & \textbf{89.64 ($\pm$0.80)} & \textbf{77.83 ($\pm$0.71)} & \textbf{67.38 ($\pm$0.46)} \\
				Hellinger    & 89.36 ($\pm$1.18) & \textbf{78.43 ($\pm$0.66)} & \textbf{68.49 ($\pm$1.07)}  \\
				\hline
			\end{tabular}
		\end{minipage}
	\end{center}
\end{table}

\begin{figure}[h]
	\centering
	\subfloat[$\eta=0$]{\includegraphics[width=0.3\linewidth]{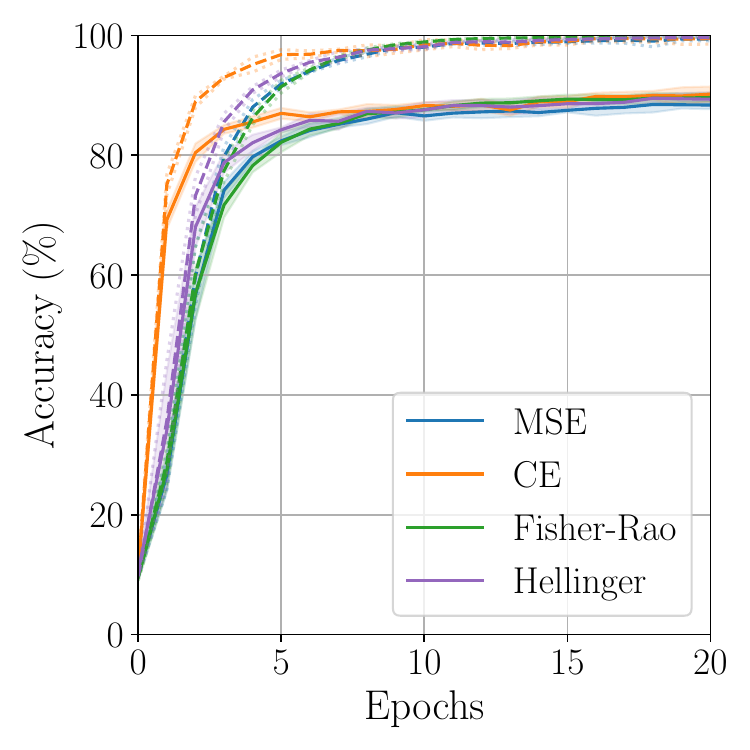}}
	\hspace{1em}
	\subfloat[$\eta=0.3$]{\includegraphics[width=0.3\linewidth]{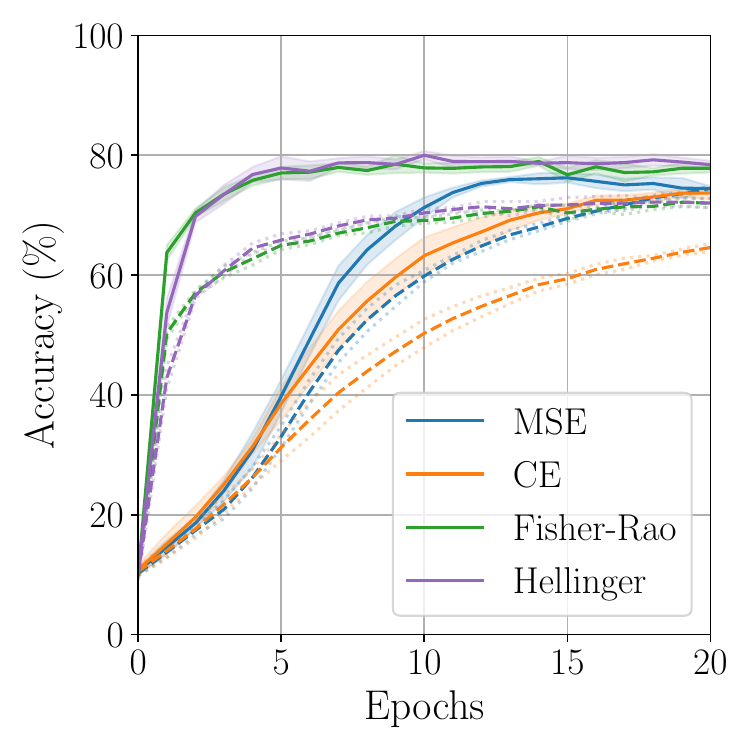}}
	\hspace{1em}
	\subfloat[$\eta=0.5$]{\includegraphics[width=0.3\linewidth]{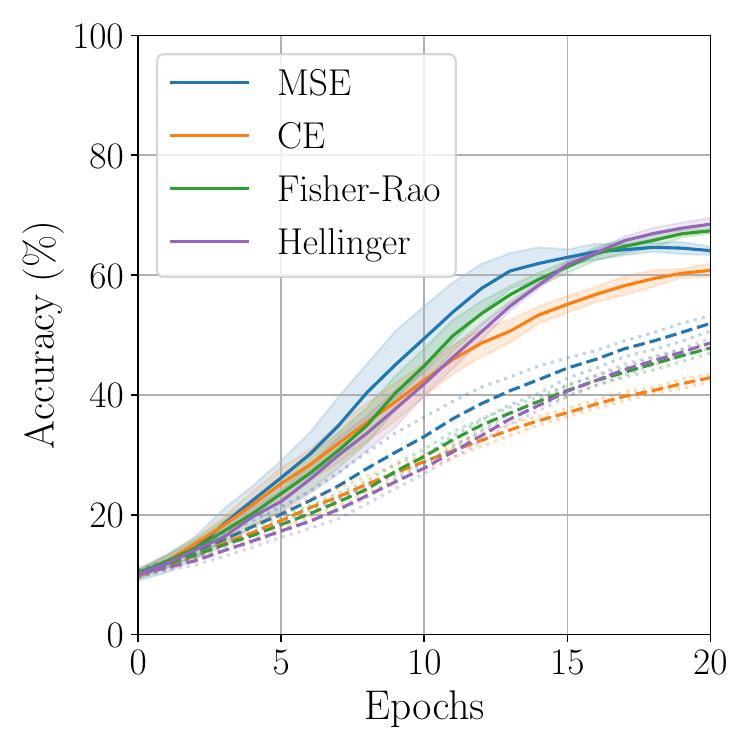}}
	\caption{Training (dashed lines) and test (solid lines) accuracy for synthetic dataset.}
	\label{fig:synthetic-results}
\end{figure}

We note that, in the absence of noise, the final accuracy achieved by all losses is similar, with the cross entropy providing the fastest training; in the presence of uniform label noise, however, the Fisher--Rao and Hellinger losses consistently provide the highest test accuracies, as expected from the analysis in Section~\ref{sec:learning}. Interestingly, they also provide faster learning than the CE loss---we conjecture this may be due to the fact that, when labels are noisy, the learning process becomes more difficult for the CE loss than for Fisher--Rao and Hellinger losses. These results emphasise that the latter two can be suitable for learning under label noise.

\subsection{MNIST}

The MNIST dataset~\cite{mnist} contains 60,000 training examples and 10,000 testing examples of $\text{28} \times \text{28}$ grey images of handwritten digits from 0 to 9. For this experiment, an MLP network is set up, consisting of two hidden layers with 300 and 100 neurons, respectively. We use batch size of~64 and train the model for~40 epochs. The experiments were repeated three times.
Results for noiseless data and label noise with $\eta=0.3$ and $\eta=0.5$ are similarly reported in Figure~\ref{fig:mnist-results} and Table \ref{tab:mnist-results}, where the two best accuracies are boldfaced.

\begin{table}[h]
	\begin{center}
		\begin{minipage}{\textwidth}
			\caption{Test accuracy (\%) for MNIST dataset.}
			\label{tab:mnist-results}
			\centering
			\setlength{\tabcolsep}{4pt}
			\renewcommand{\arraystretch}{1.15}
			\begin{tabular}{@{}ccccc@{}}
				\hline
				Loss  & Noiseless & $\eta = 0.3$ & $\eta = 0.5$ \\
				\hline
				Mean square error    & \textbf{98.41 ($\pm$0.09)} & 98.40 ($\pm$0.10) & 97.93 ($\pm$0.07) \\
				Cross entropy    & \textbf{98.50 ($\pm$0.04)} & 98.14 ($\pm$0.06) & 97.69 ($\pm$0.16) \\
				Fisher--Rao   & 98.32 ($\pm$0.07) & \textbf{98.44 ($\pm$0.05)} & \textbf{98.34 ($\pm$0.14)} \\
				Hellinger    & 98.33 ($\pm$0.05) & \textbf{98.53 ($\pm$0.03)} & \textbf{98.40 ($\pm$0.06)} \\
				\hline
			\end{tabular}
		\end{minipage}
	\end{center}
\end{table}

\begin{figure}[h]
	\centering
	\subfloat[$\eta=0$]{\includegraphics[width=0.3\linewidth]{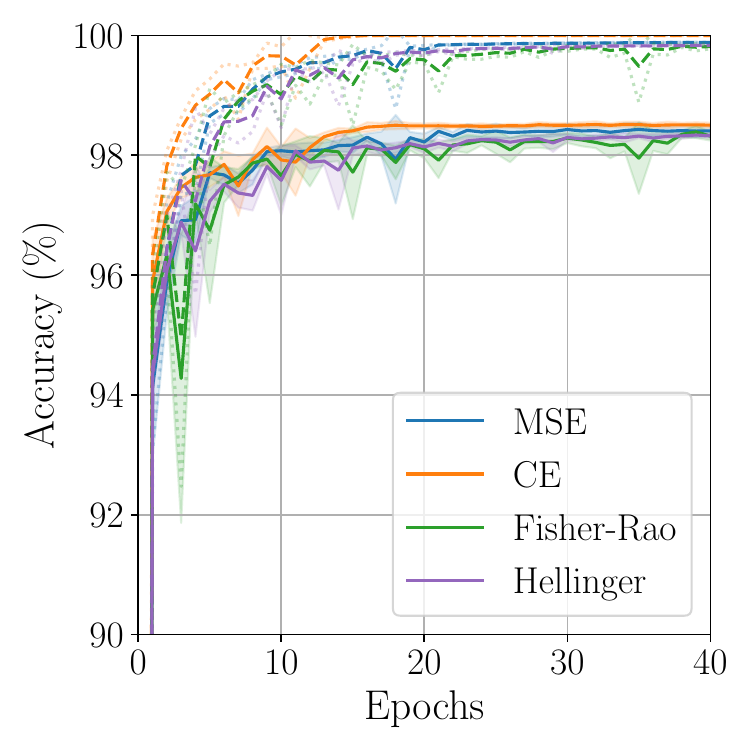}}
	\hspace{1em}
	\subfloat[\label{fig:mnist-b}$\eta=0.3$]{\includegraphics[width=0.3\linewidth]{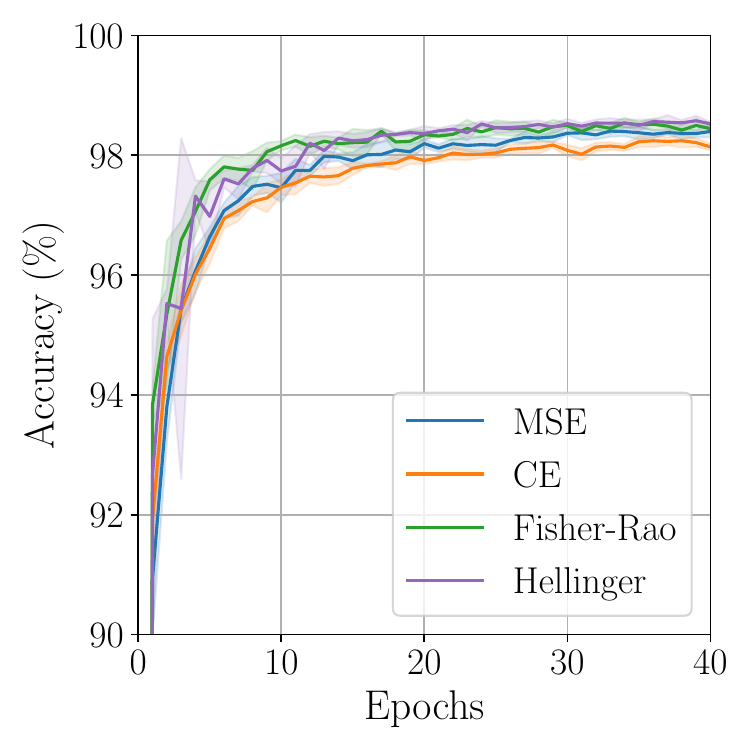}}
	\hspace{1em}
	\subfloat[\label{fig-mnist-c}$\eta=0.5$]{\includegraphics[width=0.3\linewidth]{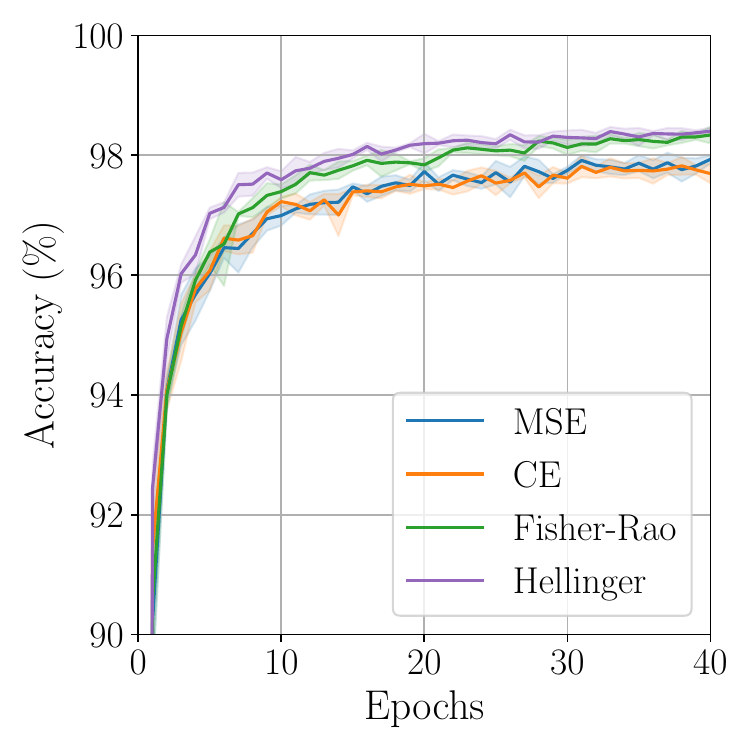}}
	\caption{Training (dashed lines) and test (solid lines) accuracy for MNIST dataset.}
	\label{fig:mnist-results}
\end{figure}

In these experiments, Fisher--Rao achieves competitive performances with noiseless data. Again, in the presence of label noise, both Fisher--Rao and Hellinger show a slight advantage over MSE and CE losses in terms of learning speed and final accuracy, illustrating the benefits discussed in the previous section. Note that in Figures~\ref{fig:mnist-b} and \ref{fig-mnist-c}, the training accuracies are not visible, as they are less than 90\%.

\section{Conclusion and Perspectives} \label{sec:conclusion}

In this work we have studied the Fisher--Rao loss function for learning in the context of supervised classification, especially in the presence of label noise. This loss comes from the Fisher metric on the statistical manifold of discrete distributions, and, for small values, is approached both by the Hellinger and the cross entropy loss functions.
Theoretical analysis of the robutstness and learning speed of this loss have been presented and compared with other commonly used losses. We argue that the Fisher--Rao loss provides a natural trade-off between these two features.
In numerical experiments, we have observed that, for noiseless data, the performance of the Fisher--Rao loss is on par with other commonly used losses, while, in the presence of label noise, we observe the expected robustness.

While we have focused on the classification problem, one could ask whether it is possible to use similar Fisher--Rao loss functions to regression problems as well. 
The first requirement to do so is that the model (e.g., neural network) estimates a (parametric) conditional distribution $p(y\vert\pmb{x})$ of the output $y \in \R$, given the input $\pmb{x} \in \R^n$, instead of a simple predicted value $\hat{y}$. For instance, this is done for normal distributions in~\cite{nix1994}, and for mixture models in~\cite[§~5.6]{bishop2006}.
The second requirement is a closed-form expression for the Fisher--Rao distance in the corresponding statistical manifold.
Finally, it is required to define a target distribution from the training examples, to be compared with the current predicted distribution, in such a way that the distance between them is well defined.
The consequences of such an approach remain yet to be studied and constitute a possible future direction of research.

\section*{Acknowledgments}

This work was partly supported by S\~{a}o Paulo Research Foundation~(FAPESP) grant 2021/04516-8 and by Brazilian National Council for Scientific and Technological Development~(CNPq) grants 141407/2020-4 and 314441/2021-2.

\begin{appendices}
	
\section{Performance Bounds for the MSE Loss} \label{ap:MSE-loss}

\begin{lemma} \label{lemma:bounds-mse}
	The mean squared error loss $L_{\MSE}$ satisfies
	\begin{equation}
		K-1 \le \sum_{i=1}^{K} L_\MSE(i, f(\pmb{x})) \le 2(K-1).
	\end{equation}
\end{lemma}

\begin{proof}
	From \eqref{eq:mse-loss}, we have $L_{\MSE}(y,f(\pmb{x})) = 1 + \|\pmb{p}\|_2^2 - 2p_y$, so that
	\begin{equation*}
		\sum_{i=1}^{K} L_{\MSE}(i,f(\pmb{x}))
		= \sum_{i=1}^{K} \left( 1 + \|\pmb{p}\|_2^2 - 2p_i \right)
		= K\left(1 + \|\pmb{p}\|_2^2\right) - 2
	\end{equation*}
	
	\noindent
	Furthermore, $F(\pmb{p}) \coloneqq \|\pmb{p}\|_2^2 = \sum_{i=1}^{K}p_i^2$ is a convex function in $\pmb{p} = (p_1, \dots, p_K)$, and, with $\pmb{p}$ restricted to the simplex $\Delta^{K-1}$, we have $\frac{1}{K} \le \|\pmb{p}\|_2^2 \le 1$. Thus
	
	\begin{equation*}
		K-1
		= K \left( 1 + \frac{1}{K} \right) - 2
		\le	\sum_{i=1}^{K} L_{\MSE}(i,f(\pmb{x}))
		\le K \left( 1 + 1 \right) - 2
		= 2(K-1).
	\end{equation*}
\end{proof}

\begin{proposition} \label{thm:bounds-mse}
	For the mean squared error loss function, under uniform label noise with $\eta < \frac{K-1}{K}$, we have
	\begin{equation}\label{eq:limitante-A-MSE}
		0 \le R^\eta_{L_\MSE}(f^*) - R^\eta_{L_\MSE}(\hat f) \le A_{\MSE}
	\end{equation}
	
	\noindent
	and
	\begin{equation}\label{eq:limitante-B-MSE}
		B_{\MSE} \le R_{L_\MSE}(f^*) - R_{L_\MSE}(\hat f) \le 0,
	\end{equation}
	
	\noindent
	where $A_{\MSE} \coloneqq A_{\MSE}(\eta) \coloneqq \eta$ and $B_\MSE \coloneqq B_\MSE(K,\eta) \coloneqq -\eta \frac{K-1}{K-1-\eta K}$.
\end{proposition}

\begin{proof}
	The proof is analogous to that of Proposition~\ref{prop:bounds-fr}, so we present a concise version. We have
	\begin{equation*}
		R^\eta_{L_\MSE}(f)
		=  \left( 1 - \frac{\eta K}{K-1} \right) R_{L_\MSE} (f)
		+ \frac{\eta}{K-1} \Exp_{\pmb{x}} \Exp_{y \rvert \pmb{x}} \left[ \sum_{i=1}^K L_{\MSE } (i, f(\pmb{x})) \right].
	\end{equation*}
	
	\noindent
	Using Lemma~\ref{lemma:bounds-mse}, we obtain
	\begin{equation*}
		R_{L_\MSE}^{\eta}(f^*) - R_{L_\MSE}^{\eta}(\hat{f})
		\le \left( 1 - \frac{\eta K}{K-1} \right) \left( R_{L_\MSE}(f^*) - R_{L_\MSE}(\hat{f}) \right) + A_\MSE
	\end{equation*}
	and
	\begin{equation*}
		R_{L_\MSE}(f^*) - R_{L_\MSE}(\hat{f})
		\ge \frac{\left(K-1\right) \left( R_{L_\MSE}^\eta(f^*) - R_{L_\MSE}^\eta(\hat{f}) \right)}{K-1-\eta K} + B_\MSE.
	\end{equation*}
	
	\noindent
	Finally, using that $R_{L_\MSE}(f^*) \le R_{L_\MSE}(\hat{f})$, $R^\eta_{L_\MSE}(\hat{f}) \le R^\eta_{L_\MSE}(f^*)$, and that ${\eta< \frac{K-1}{K}}$ yields the desired results.
\end{proof}

\end{appendices}

\bibliographystyle{ieeetr}
\bibliography{references}

\end{document}